\documentclass[11pt]{article}

\usepackage[utf8]{inputenc}
\usepackage[T1]{fontenc}
\usepackage{lmodern}

\usepackage{amsmath}
\usepackage{amssymb}
\usepackage{amsfonts}
\usepackage{mathtools}
\usepackage{amsthm}

\usepackage{graphicx}
\usepackage{booktabs}
\usepackage{array}
\usepackage{adjustbox}
\usepackage{enumitem}
\usepackage{microtype}
\usepackage{natbib}
\usepackage{xcolor}
\usepackage{url}

\usepackage{bm}

\usepackage{tabularx}
\usepackage{placeins}

\usepackage{tikz}
\usetikzlibrary{arrows.meta,positioning,calc,shapes.geometric,decorations.pathreplacing}
\usepackage{pgfplots}
\pgfplotsset{compat=1.18}
\usepackage{caption}
\usepackage{subcaption}

\usepackage{hyperref}

\usepackage{cleveref}

\graphicspath{{figures/}}
\usepackage[normalem]{ulem}
\usepackage{booktabs}
\usepackage{enumitem}

\IfFileExists{results/orbit_numbers.tex}{% Generated by code/orbit_audit.py; do not edit by hand.
\newcommand{\numOrbitSynthD}{256}
\newcommand{\numOrbitSynthStates}{32}
\newcommand{\numOrbitSynthAudits}{1000}
\newcommand{\numOrbitSynthTreeEdges}{31}
\newcommand{\numOrbitSynthCycleEdges}{32}
\newcommand{\numOrbitSynthCompleteEdges}{496}
\newcommand{\numOrbitSynthRank}{31}
\newcommand{\numOrbitSynthEmpiricalMeanRatio}{1.1362}
\newcommand{\numOrbitSynthTheoryMeanRatio}{1.1390}
\newcommand{\numOrbitDigitsSeeds}{5}
\newcommand{\numOrbitDigitsSplits}{3}

\newcommand{\numOrbitDigitsStates}{16}
\newcommand{\numOrbitDigitsPairBudget}{480}
\newcommand{\numOrbitDigitsAllPairs}{120}
\newcommand{\numOrbitDigitsHiddenWidth}{512}
\newcommand{\numOrbitDigitsGateAcceptPct}{0}
\newcommand{\numOrbitDigitsIdentityAuditRank}{377.8667}
\newcommand{\numOrbitDigitsGrayAuditRank}{369.8667}
\newcommand{\numOrbitDigitsCubeAuditRank}{225}
\newcommand{\numOrbitDigitsCompleteAuditRank}{60}
\newcommand{\numOrbitNullspaceMax}{\ensuremath{2.0\times 10^{-15}}}
\newcommand{\numOrbitBetaMinP}{0.0286}
\newcommand{\numOrbitPoincareMinEig}{\ensuremath{-9.3\times 10^{-16}}}
\newcommand{\numOrbitCycleLeakageMax}{\ensuremath{5.8\times 10^{-15}}}
\newcommand{\numOrbitWoodburyMaxRelErr}{\ensuremath{2.9\times 10^{-16}}}
\newcommand{\numOrbitGateCoveragePct}{99.8300}
\newcommand{\numOrbitGateTrials}{\ensuremath{1.0\times 10^{4}}}
\newcommand{\numOrbitGateDriftClip}{4}

\newcommand{\numOrbitDigitsIdentityDeployedAccuracy}{0.9637}

\newcommand{\numOrbitDigitsIdentityExactDrift}{\ensuremath{4.9\times 10^{-5}}}

\newcommand{\numOrbitDigitsIdentityExactAccuracy}{0.4964}

\newcommand{\numOrbitDigitsCubeVsIdentityRidgeDriftDelta}{-0.1880}
\newcommand{\numOrbitDigitsCubeVsIdentityRidgeDriftDeltaLo}{-0.2949}
\newcommand{\numOrbitDigitsCubeVsIdentityRidgeDriftDeltaHi}{-0.0969}
}{}

\theoremstyle{plain}
\newtheorem{theorem}{Theorem}
\newtheorem{proposition}[theorem]{Proposition}
\newtheorem{corollary}[theorem]{Corollary}

\newcommand{\R}{\mathbb R}
\newcommand{\E}{\mathbb E}
\newcommand{\Pp}{\mathbb P}

\newcommand{\risk}{\mathsf R}

\newcommand{\LOOO}{\textsc{looo}}

\setlist{itemsep=1pt,topsep=2pt,leftmargin=1.4em}
\hypersetup{
    colorlinks=true,
    linkcolor=blue,
    citecolor=blue,
    urlcolor=blue,
    pdftitle={Interpolation Is Not Invariance:\\
Pair Count Is Not Coverage in Transformation Audits},
    pdfauthor={Mohammed AHNOUCH and Lotfi ELAACHAK}
}

\title{Interpolation Is Not Invariance:\\
Pair Count Is Not Coverage in Transformation Audits}
\author{
    Mohammed AHNOUCH
    \\
    Université Paris 1
    \\
    Paris, France
    \and
    Lotfi ELAACHAK
    \\
    Faculty of Science and Technology of Tangier
    \\
    Abdelmalek Essaadi University
    \\
    Tangier, Morocco
}

\date{}

\begin{document}

\raggedbottom
\maketitle

\begin{abstract}
Counting equivalent pairs is the usual way to report how thoroughly a
transformation audit covers a model, and it overstates what the audit
constrains.  Pairs drawn from one semantic object are correlated, and on a
complete orbit graph most of them are algebraically redundant.  We therefore
separate four quantities that a raw count conflates: edge count $m$, effective
contrast rank $s$, population support rank $r$, and graph spectral gap $\eta$.
In a rank-$r$ Gaussian contrast model, a calibrated reader that is invariant on
the population exists exactly when the anchor has a component in $\ker T$.
Failing that, an audit of rank $s<r$ pays an unseen risk of $\risk_\star/U$ with
$U\sim\operatorname{Beta}((r-s)/2,s/2)$, and once $s\ge r$ calibrated exact
interpolation is infeasible.  Orbit topology behaves the same way: a spanning
tree imposes the exact-null constraints that a complete graph does, and a sharp
graph Poincar\'e inequality carries edge drift to the whole orbit only after
paying $1/\eta$.  Cycles can even let pair-level leave-one-out report zero error
while no semantic object was ever held out.  We therefore give exact
block-Woodbury leave-one-orbit-out updates, and a source-disjoint deployment
gate over finitely many candidates that keeps the original reader unless
uncertainty bounds certify lower drift inside a clean-utility budget.  Synthetic
and frozen-digit audits then separate nominal pair count from algebraic rank and
spectral coverage.
\end{abstract}

\section{Introduction}

Invariance audits exist because a model that answers differently on two inputs its
owner declared equivalent is a liability, and they report pair count as coverage.
Such an audit rarely contains independent pairs: it usually starts from
one semantic object, an image, intent, patient, or scene, and compares several
transformations of that object.  Record every edge of the resulting orbit graph
and a $q$-state orbit contributes $q(q-1)/2$ pairs, yet those pairs can encode
only $q-1$ independent contrasts.  A complete graph can therefore make an audit
look far larger than the representation it constrains, while adding no exact
constraint that a spanning tree does not already impose.

The gap matters as soon as a frozen representation $h$ feeds a linear reader.
Such a reader can disagree on two inputs that were declared equivalent, and the
tempting repair is to leave $h$ alone and adjust the reader until every audited
contrast vanishes.  Paired translations and paraphrases motivate the problem
\citep{yong2024lowresource,deng2024multilingual,wang2024languages}, and
controlled image experiments isolate its statistical geometry.  Driving the edge
residual to zero, however, is interpolation and nothing more: invariance asks for
small drift on a fresh semantic object, and on transformation states that no
audited edge ever joined.

Pair count alone therefore does not measure how strong an audit is.  Three
quantities pull apart, the edge count $m$, the effective contrast rank $s$, and
the population support rank $r$, and a fourth, the graph spectral gap $\eta$,
decides separately whether local comparisons certify an entire orbit.
\Cref{fig:orbit-claim} summarizes the separation.

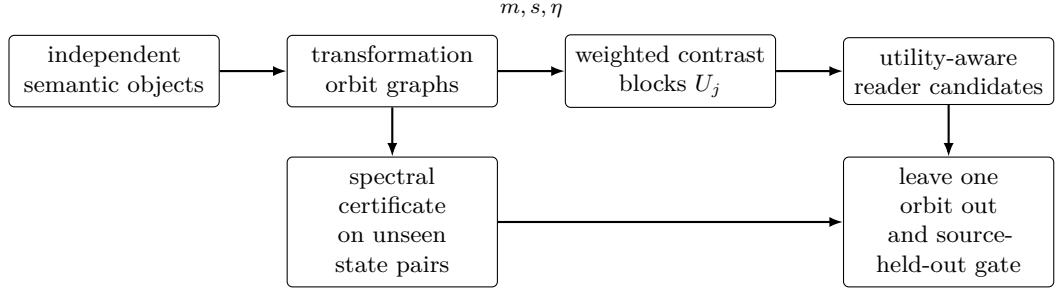
\begin{figure}[t]
\centering
\begin{tikzpicture}[
  box/.style={draw,rounded corners=2pt,align=center,inner sep=4pt,
              text width=25mm,font=\footnotesize},
  arr/.style={-{Latex[length=1.8mm]},thick}]
  \node[box] (objects)    at (0mm,0mm)    {independent\\semantic objects};
  \node[box] (graphs)     at (36.8mm,0mm) {transformation\\orbit graphs};
  \node[box] (matrix)     at (73.6mm,0mm) {weighted contrast\\blocks $U_j$};
  \node[box] (repair)     at (110.4mm,0mm){utility-aware\\reader candidates};
  \node[box] (certificate) at (36.8mm,-20mm)
        {spectral certificate\\on unseen state pairs};
  \node[box] (validation)  at (110.4mm,-20mm)
        {leave one orbit out\\and source-held-out gate};
  \draw[arr] (objects.east) -- (graphs.west);
  \draw[arr] (graphs.east)  -- (matrix.west);
  \draw[arr] (matrix.east)  -- (repair.west);
  \node[font=\scriptsize] at (55.2mm,8mm) {$m,s,\eta$};
  \draw[arr] (graphs.south) -- (certificate.north);
  \draw[arr] (repair.south) -- (validation.north);
  \draw[arr] (certificate.east) -- (validation.west);
\end{tikzpicture}
\caption{Pair count is not coverage.  Orbit topology determines exact contrast
rank; its gap converts a held-out bound on local edge drift into an all-state
certificate; and semantic objects, rather than correlated edges, are held out
for selection.}
\label{fig:orbit-claim}
\end{figure}

Three results follow.  A rank-aware finite-sample theorem shows that correlated
Gaussian audits pay the same beta-law interpolation cost as $s$ independent
constraints however large $m$ becomes, covers singular population geometry, and
says when exact invariance is available and when calibrated interpolation is
impossible.  A sharp graph Poincar\'e inequality then
turns population transition drift into an all-state orbit certificate, while a
cycle lemma explains why leaving out a single pair can be vacuous.  Finally,
utility-aware ridge, an exact block-Woodbury leave-one-orbit-out calculation and
a finite-candidate confidence gate combine into a rule that defaults to the
deployed reader.  A synthetic experiment isolates rank from edge count, and a
frozen-digits experiment compares orbit graphs at a fixed pair budget.

\noindent\textbf{Positioning.}
Nullspace projection and concept-erasure methods remove linearly recoverable
information from a representation \citep{ravfogel2020inlp,belrose2023leace,holstege2025splice}.
We instead leave the representation available to every other task and edit one
readout using within-object contrasts.  Supervised augmentation and explicit
symmetry methods can establish stronger invariance when labels or the complete
group action are available \citep{tahmasebi2026augmentation,soleymani2026symmetry};
our question is what can be certified from a finite, correlated audit.  Graph
Poincar\'e inequalities, block inverse identities, and spectral-gap design are
classical \citep{levin2017markov,chung1997spectral,ghosh2006growing,golub2013matrix}.
The contribution is their
combination into an orbit-aware audit and deployment rule, not a new graph
inequality or matrix identity.  Source-aware cross-validation is likewise
established \citep{geras2013cv}; here it prevents transformed copies of one
object from leaking across repair selection.

\section{Audits are graphs, not bags of pairs}
\label{sec:setup}

We treat an audit as one graph per object rather than a list of pairs.  Let
$x_j$, $j=1,\ldots,n$, be independent semantic objects.  Object $j$ has
transformation states $V_j$ and representations
$z_{jv}=h(g_vx_j)\in\R^d$.  An oriented weighted graph
$\mathcal G_j=(V_j,E_j)$ specifies which states are compared.  If
$B_j\in\R^{|V_j|\times |E_j|}$ is its weighted incidence matrix and
$H_j=[z_{jv}]_{v\in V_j}$, its contrast block is
\begin{equation}
  U_j=H_jB_j,\qquad X=[U_1,\ldots,U_n],\qquad
  \widehat T_G=XX^\top .
  \label{eq:orbit-design}
\end{equation}
Weights include any normalization, so deleting an object means deleting its
whole block without changing the remaining weights.  For a reader $w$, the
columns of $U_j$ contain its audited score differences.

The raw count is $m=\sum_j|E_j|$.  The effective contrast rank is the rank of
the corresponding audit design, denoted $s$ in the model below; the fresh-pair
second moment is $T\succeq0$ with support rank
$r=\operatorname{rank}(T)$.  Finally, $\eta$ is the spectral gap of the
reversible state-transition graph used for the certificate in
\cref{sec:graph}.  These quantities answer different questions:
\begin{center}
\small
\begin{tabular}{@{}llll@{}}
\toprule
quantity & meaning & increased by duplicate edges? & role \\
\midrule
$m$ & recorded comparisons & yes & compute and annotation cost \\
$s$ & independent audit contrasts & no & interpolation geometry \\
$r$ & population contrast support & no & unseen-risk dimension \\
$\eta$ & orbit connectivity/conditioning & not necessarily & coverage certificate \\
\bottomrule
\end{tabular}
\end{center}

For a fresh declared-equivalent contrast $D$, define
\begin{equation}
  T=\E[DD^\top],\qquad \risk(w)=\E(w^\top D)^2=w^\top Tw.
  \label{eq:population-risk}
\end{equation}
The deployed reader $a$ is normalized to $\|a\|=1$.  We call $a^\top w=1$
\emph{anchor calibration}.  It fixes one affine degree of freedom; by itself
it neither preserves decisions nor bounds clean accuracy.  Those properties
are measured separately below.  An exactly interpolating reader satisfies
$X^\top w=0$ in addition to anchor calibration.

\section{The effective-rank law}
\label{sec:rank-law}

The following model separates raw columns from their algebraic content.  Let
$Z\in\R^{d\times k}$ have independent standard Gaussian entries, let
$A\in\R^{k\times m}$ be deterministic with rank $s$, and set
\begin{equation}
                 D_{\rm aud}=T^{1/2}ZA.
                 \label{eq:correlated-audit}
\end{equation}
The $m$ columns may be strongly correlated.  Only the $s$-dimensional column
space selected by $A$ enters the null constraint.  An incidence matrix is the
canonical example.

\begin{theorem}[Rank-aware law for exact audit interpolation]
\label{thm:rank-beta}
Let $T\succeq0$ be deterministic with rank $r$, let $A$ and $Z$ be as in
\cref{eq:correlated-audit}, and let $a_0=P_{\ker T}a$.
\begin{enumerate}[label=(\roman*),leftmargin=1.7em]
\item If $a_0\ne0$, then $w_0=a_0/\|a_0\|^2$ is anchor-calibrated, annihilates
the audit, and has $\risk(w_0)=0$.
\item If $a_0=0$ and $1\le s<r$, then almost surely
\begin{equation}
 \min_{D_{\rm aud}^\top w=0,\;a^\top w=1}\risk(w)
 =\frac{\risk_\star}{U},\qquad
 U\sim\operatorname{Beta}\!\left(\frac{r-s}{2},\frac{s}{2}\right),
 \quad \risk_\star=(a^\top T^\dagger a)^{-1}.
 \label{eq:rank-beta}
\end{equation}
For $s=0$, the same statement holds with $U\equiv1$.
\item If $a_0=0$ and $s\ge r$, no anchor-calibrated exact interpolator exists
almost surely.
\end{enumerate}
In case (ii), for every anchor-calibrated exact audit interpolator and
$\varrho\ge1$,
\begin{align}
 \Pp\{\risk(w)>\varrho\risk_\star\}
 &\ge I_{1/\varrho}\!\left(\frac{r-s}{2},\frac{s}{2}\right),
 \label{eq:rank-tail}\\
 \E\!\left[\min\risk(w)\right]
 &=\risk_\star\frac{r-2}{r-s-2},\qquad s\le r-3,
 \label{eq:rank-mean}
\end{align}
where $I_x$ is the regularized incomplete beta function and the mean is
infinite for $s\in\{r-2,r-1\}$.  Along a sequence with $r,s\to\infty$ and
$r/s\to\gamma>1$, the ratio of the minimum to $\risk_\star$ converges in
probability to $\gamma/(\gamma-1)$.
\end{theorem}

\begin{proof}
Part (i) follows from $Ta_0=0$ and
$a^\top a_0=\|a_0\|^2$.  For the other cases, write
$T=V\Lambda V^\top$, where $V\in\R^{d\times r}$ has orthonormal columns and
$\Lambda\succ0$.  Since $a_0=0$, put
$b=\Lambda^{-1/2}V^\top a$ and
$y=\Lambda^{1/2}V^\top w$.  Kernel components of $w$ affect neither the
constraints nor the objective, and
\[
 \risk(w)=\|y\|^2,\qquad a^\top w=b^\top y,
 \qquad \|b\|^2=a^\top T^\dagger a.
\]
Take a thin singular-value decomposition $A=L\Sigma R^\top$.  Rotational
invariance makes $V^\top ZL$ a standard Gaussian $r\times s$ matrix, while
$\Sigma R^\top$ does not change its column span.  Thus interpolation is
$y\perp\mathcal S$ for a Haar-random $s$-plane $\mathcal S\subset\R^r$.
When $s<r$, Cauchy--Schwarz on $\mathcal S^\perp$ gives the unique minimum-risk
coordinate
\[
 y_\star=\frac{P_{\mathcal S^\perp}b}
                 {\|P_{\mathcal S^\perp}b\|^2},\qquad
 \|y_\star\|^2=\frac{1}{\|b\|^2U},\quad
 U=\frac{\|P_{\mathcal S^\perp}b\|^2}{\|b\|^2}.
\]
The squared projection of a fixed direction onto a Haar $(r-s)$-plane is the
beta variable in \cref{eq:rank-beta} \citep{muirhead1982aspects}.  Its inverse
tail and first inverse moment give \cref{eq:rank-tail,eq:rank-mean}; beta
concentration gives the limit.  If $s\ge r$, the audit span equals $\R^r$
almost surely, forcing $y=0$, which contradicts $b^\top y=1$.
\end{proof}

The kernel case is a genuine population-invariant escape, but it may move far
from the deployed reader; this is precisely why anchor calibration is not a
utility guarantee.  In the usual case $a\in\operatorname{range}(T)$, the
theorem replaces the nominal sample size $m$ by $s$.  In particular, when
$1\le s<r$, for $\risk_0=a^\top Ta$ every anchor-calibrated exact audit
interpolator is worse than deployment with probability at least
\begin{equation}
 I_{\risk_\star/\risk_0}\!\left(\frac{r-s}{2},\frac{s}{2}\right),
 \label{eq:worse-than-deployed}
\end{equation}
because $(a^\top Ta)(a^\top T^\dagger a)\ge1$ for normalized
$a\in\operatorname{range}(T)$.

\begin{corollary}[Edge redundancy]
\label{cor:incidence-rank}
If $B$ is an incidence matrix of a graph with $q$ vertices and $c$ connected
components, then $\operatorname{rank}(B)=q-c$.  A complete graph and any
spanning tree on the same connected orbit therefore impose identical
exact-null constraints $B^\top H^\top w=0$, despite having different edge
counts.
\end{corollary}
\begin{proof}
The kernel of $B^\top$ consists of vectors constant on each component, hence
has dimension $c$.  For any two connected graphs, the incidence columns span
the same zero-sum subspace; applying the same linear map $H$ preserves equality
of their column spans.
\end{proof}

Exact constraints depend only on connected components.  Approximate repair is
different: edge weights and graph conditioning decide how small observed
edge energy controls the unobserved orbit.

\section{From local edges to orbit coverage}
\label{sec:graph}

Edges are local while the claim we want covers every state.  Fix an object $x$
and write $f_x(v)=w^\top h(g_vx)$.  Let $P$ be a reversible
Markov kernel on a finite set of transformation states with stationary law
$\pi$.  Draw
$V,V'\stackrel{\rm iid}{\sim}\pi$ and $W\mid V\sim P(V,\cdot)$, and average
also over fresh objects:
\begin{align}
 \risk_{\rm orbit}(w)&=\E_{x,V,V'}[f_x(V)-f_x(V')]^2,
 \label{eq:orbit-risk}\\
 \risk_{\rm edge}(w)&=\E_{x,V,W}[f_x(V)-f_x(W)]^2.
 \label{eq:edge-risk}
\end{align}
Let $\eta$ be the Poincar\'e spectral gap of $P$ on mean-zero functions
\citep{levin2017markov,chung1997spectral}.

\begin{theorem}[Sharp orbit certificate]
\label{thm:graph-poincare}
If the transformation graph is connected and reversible, then every reader
satisfies
\begin{equation}
                   \risk_{\rm orbit}(w)
                   \le \frac{\risk_{\rm edge}(w)}{\eta}.
                   \label{eq:graph-certificate}
\end{equation}
The constant $1/\eta$ is sharp.
\end{theorem}
\begin{proof}
For fixed $x$, let $\bar f_x=\E_\pi f_x$.  Independence gives
$\E_{V,V'}(f_x(V)-f_x(V'))^2=2\operatorname{Var}_\pi(f_x)$.  Reversibility
gives
$\E_{V,W}(f_x(V)-f_x(W))^2=2\langle f_x,(I-P)f_x\rangle_\pi$.
The variational definition of the gap yields
$\eta\operatorname{Var}_\pi(f_x)\le
\langle f_x,(I-P)f_x\rangle_\pi$.  Multiply by two and average over $x$.
An eigenfunction associated with the first nonzero eigenvalue attains equality.
\end{proof}

The certificate can be made finite-sample at the same semantic-object level.
Conditional on the training data, fix the reader and let
$L_j^{\rm edge}\in[0,B]$ be validation object $j$'s expected transition loss
over the chosen graph and let
$\widehat R_{\rm edge}=n^{-1}\sum_jL_j^{\rm edge}$.  Hoeffding's inequality
and \cref{thm:graph-poincare} give, with probability at least $1-\delta$,
\begin{equation}
 \risk_{\rm orbit}(w)\le
 \frac{\widehat R_{\rm edge}
       +B\sqrt{\log(1/\delta)/(2n)}}{\eta}.
 \label{eq:empirical-certificate}
\end{equation}
A small gap magnifies both observed edge drift and statistical uncertainty.
If the same objects influenced the choice of $w$, a separate validation split
or the simultaneous bound in \cref{thm:finite-gate} is required.

Disconnected audits have $\eta=0$: their edge loss cannot control offsets
between components.  In particular, checking each generator only at the
identity does not establish functional invariance under compositions.  A
generator edge must be sampled at orbit locations, as in a connected Cayley
graph, before \cref{eq:graph-certificate} applies.  Under a fixed
normalization, maximizing the corresponding Poincar\'e gap---the Fiedler value
in a symmetric-Laplacian formulation---is a classical connectivity objective
\citep{ghosh2006growing}.  Here it supplies an audit-design rule:
first connect the declared orbit, then spend remaining comparisons where they
increase $\eta$ or reduce uncertainty.

\begin{proposition}[Cycle leakage in pair holdout]
\label{prop:cycle-leakage}
Let $e$ be an edge lying on a cycle and let $B_{-e}$ be the incidence matrix
after deleting it.  Then $b_e\in\operatorname{col}(B_{-e})$.  Consequently,
any reader that exactly annihilates $HB_{-e}$ also has zero residual on the
held-out contrast $Hb_e$.
\end{proposition}
\begin{proof}
Rescale each nonzero weighted column to its unweighted incidence vector and
orient the cycle consistently.  These signed columns sum to zero, so $b_e$ is
a weight-adjusted linear combination of the remaining cycle columns.
Multiplying by $H$ preserves the relation.
\end{proof}

Thus pair-level leave-one-out can report perfect prediction without removing
any independent information.  The appropriate deletion unit is the semantic
object and all of its orbit edges.

\section{Utility-aware repair and safe selection}
\label{sec:selection}

Repair should not cost more clean behavior than it buys in invariance.  With
$G\succ0$ measuring the edit cost, we fit
\begin{equation}
 \widehat w_{\lambda,G}
 =\arg\min_{a^\top w=1}
 \left\{w^\top XX^\top w+\lambda(w-a)^\top G(w-a)\right\}.
 \label{eq:utility-ridge}
\end{equation}
This need not lie in the exact-null plane when $\lambda>0$.  Put
$M=XX^\top+\lambda G$, $Q=M^{-1}$, and $b=\lambda Ga$.  Direct Lagrange
optimization gives
\begin{equation}
 \widehat w_{\lambda,G}=Q(b+\tau a),\qquad
 \tau=\frac{1-a^\top Qb}{a^\top Qa}.
 \label{eq:ridge-closed}
\end{equation}

For leave-one-orbit-out (\LOOO), delete the entire block $U_j$ from
\cref{eq:orbit-design}.  The inverse needed for the refit is available from
one full fit:
\begin{proposition}[Exact block-Woodbury \LOOO]
\label{prop:block-woodbury}
For $\lambda>0$, define $Q_{-j}=(M-U_jU_j^\top)^{-1}$.  Then
\begin{equation}
 Q_{-j}=Q+QU_j(I-U_j^\top QU_j)^{-1}U_j^\top Q.
 \label{eq:block-woodbury}
\end{equation}
Substituting $Q_{-j}$ for $Q$ in \cref{eq:ridge-closed} gives exactly the
reader refitted without semantic object $j$.
\end{proposition}
\begin{proof}
Because $M-U_jU_j^\top=\lambda G+\sum_{k\ne j}U_kU_k^\top\succ0$, the stated
inverse exists.  The matrix inversion lemma for a negative rank-$|E_j|$
update gives \cref{eq:block-woodbury}; the constrained minimizer then follows
from the same Lagrange calculation as \cref{eq:ridge-closed}.
\end{proof}

\LOOO{} is a useful diagnostic, but a separate source-disjoint validation set
makes model selection transparent.  Let $\mathcal C$ be $K$ candidate graph,
penalty, and utility-metric choices, all fit without validation data.  For
validation object $j$, let $\Delta^I_{jc}$ be candidate $c$'s orbit-loss minus
the deployed reader's orbit-loss, and let $\Delta^U_{jc}$ be its clean-error
increase.  Suppose their ranges have pre-specified widths $W_I,W_U$; clipping
scores before squared loss is one way to ensure this.

\begin{theorem}[Finite-candidate source-held-out gate]
\label{thm:finite-gate}
On $n_v$ independent validation objects, let
$\widehat\Delta^k_c=n_v^{-1}\sum_j\Delta^k_{jc}$ for
$k\in\{I,U\}$ and set
\begin{equation}
 \epsilon_k=W_k\sqrt{\frac{\log(4K/\delta)}{2n_v}}.
 \label{eq:hoeffding-radius}
\end{equation}
With probability at least $1-\delta$, simultaneously for every candidate and
both metrics,
$|\widehat\Delta^k_c-\E\Delta^k_c|\le\epsilon_k$.
Hence a gate that returns a repair only if
\begin{equation}
 \widehat\Delta^I_c+\epsilon_I<0,
 \qquad
 \widehat\Delta^U_c+\epsilon_U\le0.006
 \label{eq:deployment-gate}
\end{equation}
guarantees, on the same event, lower orbit loss than deployment and at most a
$0.6$-percentage-point clean-accuracy loss.  If no candidate passes, returning
$a$ preserves the baseline.
\end{theorem}
\begin{proof}
Hoeffding's inequality bounds either tail for one candidate and metric by
$\exp(-2n_v\epsilon_k^2/W_k^2)$.  A union bound over $2K$ candidate--metric
pairs and two tails gives the simultaneous event.  On it, each left side of
\cref{eq:deployment-gate} is an upper bound on its population difference, so
the two conclusions follow even after selecting a candidate adaptively.
\end{proof}

When a valid variance upper bound $\sigma_k^2$ is available, the pre-specified
Bernstein radius
\[
 \epsilon_k^{\rm B}=
 \sqrt{\frac{2\sigma_k^2\log(4K/\delta)}{n_v}}
 +\frac{2W_k\log(4K/\delta)}{3n_v}
\]
may replace \cref{eq:hoeffding-radius}.  Candidate selection uses the smallest
invariance upper bound among utility-feasible candidates, then applies the
strict improvement test; it never treats training annihilation as evidence.

\begin{center}
\begin{minipage}{0.96\linewidth}
\hrule\vspace{3pt}
\textbf{Orbit-aware readout audit (fixed before test evaluation)}
\begin{enumerate}[leftmargin=1.6em,itemsep=1pt,topsep=2pt]
\item Split semantic objects into train, validation, and test sets before
generating transformations.
\item For each candidate graph, form orbit blocks $U_j$; record $m$, $s$, and
$\eta$.  A disconnected graph cannot claim compositional coverage.
\item Fit \cref{eq:utility-ridge} over a finite penalty grid.  Compute block
\LOOO{} by \cref{eq:block-woodbury} only as a diagnostic.
\item On validation objects, evaluate all declared state pairs and clean
utility.  Apply the simultaneous gate in \cref{eq:deployment-gate}.
\item Freeze the passing candidate with the lowest orbit-loss upper bound, or
retain $a$ if none passes.  Evaluate the test objects once.
\end{enumerate}
\vspace{3pt}\hrule
\end{minipage}
\end{center}

\section{Experiments}
\label{sec:experiments}

\subsection{Rank collapse despite hundreds of edges}

We draw \numOrbitSynthAudits{} audits with
$d=\numOrbitSynthD{}$ and $q=\numOrbitSynthStates{}$ latent transformation
states under a full-rank spiked $T$.  For each latent state matrix we form a spanning
tree, a cycle, and a complete graph.  They contain respectively
\numOrbitSynthTreeEdges{}, \numOrbitSynthCycleEdges{}, and
\numOrbitSynthCompleteEdges{} recorded pairs, but each connected design has
effective rank \numOrbitSynthRank{}.  In particular, the complete audit has
more nominal pairs than representation dimensions without approaching the
$s\ge r$ infeasibility boundary.

\begin{figure}[t]
\centering
\includegraphics[width=\linewidth]{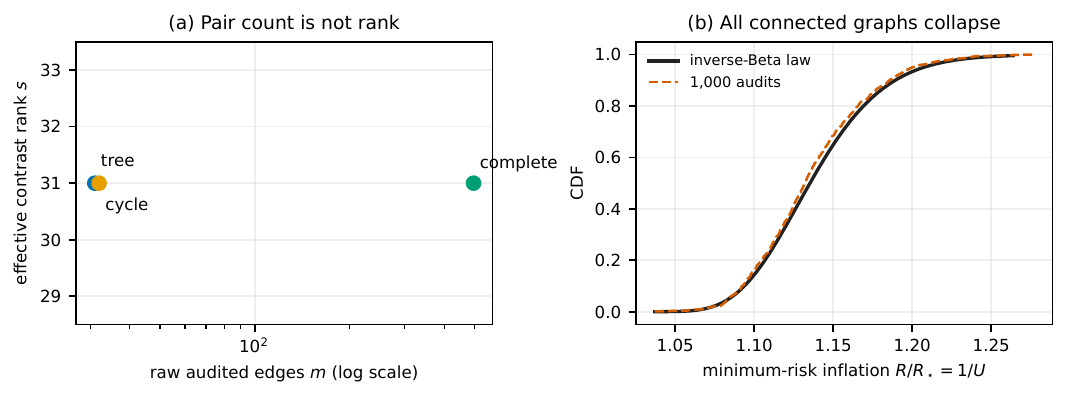}
\caption{Raw edge count collapses to graph rank.  Tree, cycle, and complete
orbits have the same exact nullspace for each audit, and their normalized
optimal risks follow $1/U$, where
$U\sim\operatorname{Beta}((\numOrbitSynthD-\numOrbitSynthRank)/2,
\numOrbitSynthRank/2)$.  Curves use \numOrbitSynthAudits{} independent
audits; theory is not fit to the simulation.}
\label{fig:rank-collapse}
\end{figure}

\Cref{fig:rank-collapse} shows the three empirical laws on top of the same
inverse-Beta prediction; the pre-specified Kolmogorov--Smirnov diagnostics are
applied to $U=\risk_\star/\min\risk$.  Their smallest $p$-value is
\numOrbitBetaMinP{}, and the empirical mean risk ratio is
\numOrbitSynthEmpiricalMeanRatio{} against the exact
\numOrbitSynthTheoryMeanRatio{}.  The maximum numerical difference between
the tree and complete exact-null projectors is \numOrbitNullspaceMax{}.  Thus
$m$ can cross $d$ while the effective constraint rank remains small.

\subsection{Frozen digits: equal budgets, different orbit graphs}
\label{sec:digits}

On the $8\times8$ optical digits \citep{alpaydin1998optical,pedregosa2011scikit},
the task is parity.  For each of \numOrbitDigitsSeeds{} seeds, we train a
one-hidden-layer ReLU network of width \numOrbitDigitsHiddenWidth{} for 180
full-batch epochs on 500 clean sources, then freeze its hidden map, normalized
parity readout $a$, and bias.  Each of \numOrbitDigitsSplits{} source
partitions reserves disjoint sets of $120/300/500/377$ images for
audit/tuning/gate/test before constructing any orbit.

Four binary state bits apply, in order, brightness $+0.08$, contrast $\times
1.15$ about intensity $0.25$, a $0.65$-pixel horizontal shift, and Gaussian
blur with width $0.65$.  Thus every source has \numOrbitDigitsStates{} states.
The clean metric is parity accuracy at the frozen zero threshold and bias.
Using only the 500 representation-training sources, we set
$G=C+0.02\operatorname{tr}(C)I/d$, with $C$ their hidden covariance.  We tune 17
penalties over a trace-scaled $10^{-4}$--$10^4$ grid on the tuning sources.
The exact-null comparator is the minimum-Euclidean-edit reader subject to
$X^\top w=0$ and $a^\top w=1$; the ridge reader uses \cref{eq:utility-ridge}.
Tuning enumerates all 120 state pairs.  On the independent gate split, squared
drift is clipped at the pre-specified $B=\numOrbitGateDriftClip{}$ only for its
signed paired upper bound; an exact binomial upper bound on newly harmed clean
predictions conservatively bounds net accuracy loss.  Test drift is unclipped.

Four audits receive the same budget of \numOrbitDigitsPairBudget{} recorded
pairs: generator checks only at the identity, a Gray-code cycle, the
four-dimensional cube, and the complete graph.  Their respective edge counts
per orbit are $4,16,32,120$, so the budget covers $120,30,15,4$ independent
training sources.  The first design leaves composition
states disconnected; the other three cover all states but have different
gaps and edge multiplicities.  For every design we fit the deployed reader,
exact nulling, and utility-aware ridge.  Ridge penalties are selected only on
source-disjoint validation data; pair leave-one-out and block \LOOO{} are
reported as diagnostics, not selection evidence.

The graph comparison is structural before any reader is fitted.  Writing
$q=\numOrbitDigitsStates{}$, the table reports the mean realized audit rank
across the \numOrbitDigitsSeeds{}$\times$\numOrbitDigitsSplits{} fits:
\begin{center}
\small
\begin{tabular}{@{}lrrrrr@{}}
\toprule
graph & $e_G$ & $n_G$ & $q-c$ & $\overline{\operatorname{rank}X}$ & $\eta$ \\
\midrule
identity generators & $4$ & $120$ & $4$ & \numOrbitDigitsIdentityAuditRank{} & $0$ \\
Gray-code cycle & $q$ & $30$ & $q-1$ & \numOrbitDigitsGrayAuditRank{} & $1-\cos(2\pi/q)$ \\
$4$-cube & $2q$ & $15$ & $q-1$ & \numOrbitDigitsCubeAuditRank{} & $1/2$ \\
complete & $q(q-1)/2$ & $4$ & $q-1$ & \numOrbitDigitsCompleteAuditRank{} & $q/(q-1)$ \\
\bottomrule
\end{tabular}
\end{center}
The fixed budget induces a breadth--connectivity trade-off.  Identity checks
cover the most objects but disconnect composition states; complete graphs give
the strongest per-orbit certificate from the fewest objects; cycles and cubes
lie between.  Neither maximum edge count nor maximum gap is therefore
guaranteed to minimize held-out drift.

Test evaluation enumerates all \numOrbitDigitsAllPairs{} distinct unordered
state pairs per source.  Under the uniform state law, multiplying this
distinct-pair mean by $15/16$ recovers \cref{eq:orbit-risk}; the factor does not
change paired graph comparisons.  Squared drift and threshold disagreement are
averaged over those pairs; clean accuracy uses state zero.  Intervals resample
networks, then source splits within networks, then source objects while keeping
graph contrasts paired at all three levels.  A graph-design advantage is called
directional only when its paired $95\%$ interval excludes zero.

\begin{figure}[t]
\centering
\includegraphics[width=\linewidth]{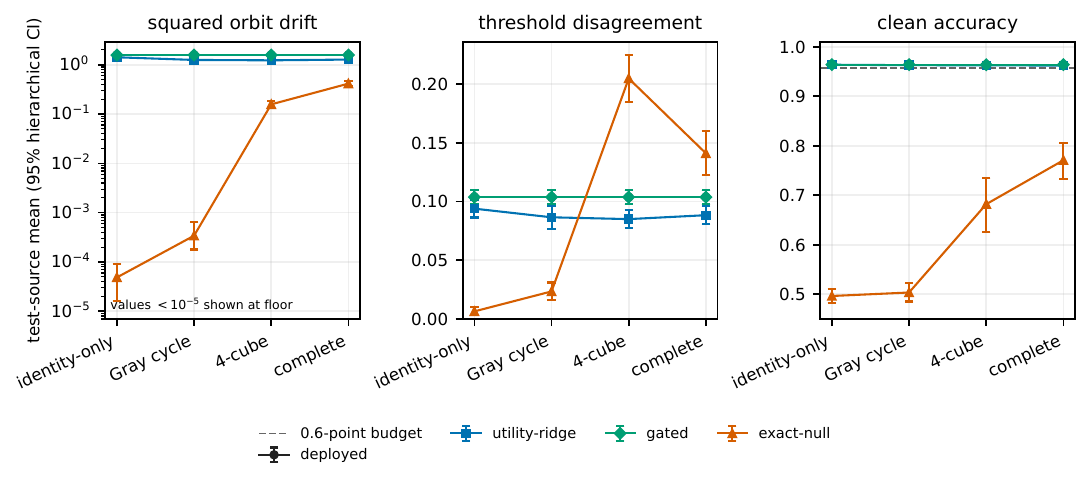}
\caption{Frozen-digits orbit audit at an equal pair budget.  Panels report
all-state squared drift, threshold disagreement, and clean accuracy; the table
reports rank and gap.  Points pair the same source split and network seed; bars
are hierarchical $95\%$ intervals.  The dashed accuracy line marks the
pre-specified $0.6$-point deployment budget.}
\label{fig:orbit-digits}
\end{figure}

The pre-specified cube-minus-identity utility-ridge contrast in
\cref{fig:orbit-digits} is directional: drift changes by
\numOrbitDigitsCubeVsIdentityRidgeDriftDelta{} (paired
hierarchical $95\%$ interval
$[\numOrbitDigitsCubeVsIdentityRidgeDriftDeltaLo{},
\numOrbitDigitsCubeVsIdentityRidgeDriftDeltaHi{}]$).  Under exact nulling,
identity-only drift falls to \numOrbitDigitsIdentityExactDrift{} while clean
accuracy falls from \numOrbitDigitsIdentityDeployedAccuracy{} to
\numOrbitDigitsIdentityExactAccuracy{}.  The joint gate accepts
\numOrbitDigitsGateAcceptPct{}\% of fits and otherwise returns deployment, so
the cube contrast does not satisfy the clean-utility deployment criterion.
Identity-only has $\eta=0$; connected graphs saturate per-orbit
rank $q-1$, but denser graphs see fewer sources and lower aggregate rank.

\subsection{What to optimize under a pair budget}

Let $M$ be a fixed number of recorded comparisons and let a candidate graph
$\mathcal G$ use $e_G$ edges per semantic object.  It can then cover roughly
$n_G=\lfloor M/e_G\rfloor$ independent objects.  More edges within an orbit can
increase $\eta_G$, but they reduce $n_G$ and hence weaken source-level
concentration.  Combining \cref{eq:empirical-certificate} with this accounting
suggests the pre-validation design score
\begin{equation}
 \mathcal B(\mathcal G)=
 \frac{\widehat R_{\rm edge}(\mathcal G)
 +B\sqrt{\log(1/\delta)/(2n_G)}}{\eta_G},
 \qquad \eta_G>0.
 \label{eq:budget-score}
\end{equation}
When several graphs share pilot data, the logarithm is adjusted for the number
of candidates.  Candidate selection is followed by the untouched validation
gate of \cref{thm:finite-gate} and one test evaluation.

Neither a spanning tree nor a complete graph uniformly minimizes
\cref{eq:budget-score}.  A tree attains the full exact-null rank with the fewest
edges, maximizing source breadth, but may have a small gap and a loose orbit
certificate.  A complete graph maximizes local connectivity but repeatedly
measures contrasts already in the same incidence span.  Intermediate graphs
can improve the gap per added edge.  A practical greedy heuristic begins with
a spanning graph and adds the edge with the largest predicted improvement in
the chosen Poincar\'e gap (or the Fiedler value under the corresponding
Laplacian normalization) \citep{ghosh2006growing}, recomputing the source-count
penalty in \cref{eq:budget-score} after each addition.  It stops when another
edge would worsen the bound; the stopping rule has no global optimality
guarantee.  Once $s$ is saturated within an orbit, added edges must improve spectral conditioning
rather than merely increase the recorded pair count.

\noindent\textbf{Mechanical validation.}
Graph-incidence ranks and the equality of tree and complete nullspaces hold to
maximum discrepancy \numOrbitNullspaceMax{}.  The minimum eigenvalue in the numerical
Poincar\'e matrix-slack check is \numOrbitPoincareMinEig{} (up to floating-point
tolerance).  The maximum held-cycle residual after exact nulling is
\numOrbitCycleLeakageMax{}.  Block-Woodbury \LOOO{} agrees with brute-force
refitting to relative error \numOrbitWoodburyMaxRelErr{}, below the
pre-specified $10^{-9}$ tolerance.  At $\delta=0.05$, the source-level gate's
simultaneous event holds in \numOrbitGateCoveragePct{}\% of
\numOrbitGateTrials{} repetitions (nominal $95\%$).

\section{Limitations and conclusion}

The beta law assumes Gaussian latent innovations and a fixed $A$; the graph
certificate is distribution-free but needs a correct equivalence relation and a
finite reversible graph; a frozen linear head cannot recover information the
representation never carried; anchor calibration is normalization and nothing
more; and the gate needs independent source orbits with bounded loss.  The digit
transformations are mild and in fixed order, so they do not probe noncommuting
actions or a misdeclared relation, and digits test orbit geometry rather than
multilingual guards or foundation models.

What an audit certifies is fixed by its rank and its orbit topology.  Count
semantic objects rather than edges, report the contrast rank, connect the orbit
and report its gap, hold out whole orbits, and deploy only when a source-level
bound beats the deployed reader inside the budget.  Pair count alone certifies
none of these.

\label{end:main}
\clearpage

{\small
\setlength{\bibsep}{2pt plus 0.3ex}
\bibliographystyle{plainnat}
\bibliography{references}}

\end{document}